\documentclass[twoside,11pt,tablecaption=bottom]{jmlr}

\usepackage{times}

\usepackage{mathtools}
\usepackage{enumitem}
\newlist{inlistalph}{enumerate*}{1}
\setlist[inlistalph]{label=(\alph*)}

\usepackage[utf8]{inputenc}
\usepackage[english]{babel}
\usepackage{amsopn}
\usepackage{amsfonts}
\usepackage{mathrsfs}
\usepackage{xparse}
\usepackage{upgreek}
\usepackage{xcolor}

\usepackage{wasysym}
\usepackage{color}
\usepackage{tabu}

\usepackage{float}

\usepackage{verbatim}

\usepackage{adjustbox}

\newcommand*{\qedhere}{\hfill\BlackBox\\[2mm]}
\newcommand*{\noqed}{\renewcommand{\jmlrQED}{}}

\usepackage{tikz}
\usepackage{xspace}
\usetikzlibrary{shapes, fit, decorations.pathreplacing, shadows, arrows, calc, positioning}

\usepackage{pgfplots}
\usepgfplotslibrary{patchplots}

\newcommand*{\Txt}{\mathbf{Txt}}

\newcommand*{\G}{\mathbf{G}}

\newcommand*{\It}{\mathbf{It}}
\newcommand*{\Sd}{\mathbf{Sd}}
\newcommand*{\Psd}{\mathbf{Psd}}

\newcommand*{\Ex}{\mathbf{Ex}}
\newcommand*{\Bc}{\mathbf{Bc}}

\newcommand*{\Mon}{\mathbf{Mon}}
\newcommand*{\SMon}{\mathbf{SMon}}

\newcommand*{\T}{\mathbf{T}}

\newcommand*{\N}{\mathbb{N}}

\newcommand*{\La}{\mathcal{L}}

\newcommand*{\totalCp}{\mathcal{R}}
\newcommand*{\partialCp}{\mathcal{P}}

\newcommand{\dom}{\mathrm{dom}}
\newcommand{\range}{\mathrm{range}}
\newcommand{\content}{\mathrm{content}}

\newcommand{\first}{\mathrm{first}}

\newcommand{\ind}{\mathrm{ind}}
\newcommand{\pad}{\mathrm{pad}}
\newcommand{\unpad}{\mathrm{unpad}}

\newcommand{\join}{\mathrm{join}}

\newcommand{\ORT}{\textbf{ORT}\xspace}

\makeatletter
\newsavebox{\@brx}
\newcommand{\llangle}[1][]{\savebox{\@brx}{\(\m@th{#1\langle}\)}%
  \mathopen{\copy\@brx\kern-0.5\wd\@brx\usebox{\@brx}}}
\newcommand{\rrangle}[1][]{\savebox{\@brx}{\(\m@th{#1\rangle}\)}%
  \mathclose{\copy\@brx\kern-0.5\wd\@brx\usebox{\@brx}}}
\makeatother

\newcommand{\convs}{\mathclose{\hbox{$\downarrow$}}}
\newcommand{\divs}{\mathclose{\hbox{$\uparrow$}}}

\usetikzlibrary{fadings,through}

\newcommand*{\concat}{^\frown}

\newcommand*{\falls}{\text{if }}
\newcommand*{\sonst}{\text{otherwise.}}
\newcommand*{\sonstfalls}{\text{else, if }}

\newcommand*{\Lr}{\mathbf{LR}}

\makeatletter 
\g@addto@macro{\@algocf@init}{\SetKwInput{input}{Input}} 
\makeatother
\makeatletter 
\g@addto@macro{\@algocf@init}{\SetKwInput{outoutput}{Output}} 
\makeatother
\makeatletter 
\g@addto@macro{\@algocf@init}{\SetKwInput{param}{Parameter}} 
\makeatother
\makeatletter 
\g@addto@macro{\@algocf@init}{\SetKwInput{output}{Semantic Output}} 
\makeatother
\makeatletter 
\g@addto@macro{\@algocf@init}{\SetKwInput{init}{Initialization}} 
\makeatother

\newcommand{\itemin}[1]{\item[#1\hspace{-0.5cm}] \hspace{0.5cm}}

\jmlrvolume{}
\jmlryear{}
\jmlrsubmitted{}
\jmlrpublished{}
\jmlrworkshop{}
\jmlrpages{}
\jmlrproceedings{}{}

\title[Mapping Monotonic Restrictions in Inductive Inference]{Mapping Monotonic Restrictions in Inductive Inference}

\author{\Name{Vanja Dosko\v{c}} \Email{vanja.doskoc@hpi.de} \\ \Name{Timo K\"{o}tzing} \Email{timo.koetzing@hpi.de} \\ \addr Hasso Plattner Institute \\ University of Potsdam, Germany}

\pgfplotsset{compat=1.16}

\begin{document}

\maketitle

\begin{abstract}
  In \emph{language learning in the limit} we investigate computable devices (learners) learning formal languages. Through the years, many natural restrictions have been imposed on the studied learners. As such, \emph{monotonic} restrictions always enjoyed particular attention as, although being a natural requirement, monotonic learners show significantly diverse behaviour when studied in different settings. A recent study thoroughly analysed the learning capabilities of \emph{strongly monotone} learners imposed with memory restrictions and various additional requirements. The unveiled differences between \emph{explanatory} and \emph{behaviourally correct} such learners motivate our studies of \emph{monotone} learners dealing with the same restrictions.

  We reveal differences and similarities between monotone learners and their strongly monotone counterpart when studied with various additional restrictions. In particular, we show that explanatory monotone learners, although known to be strictly stronger, do (almost) preserve the pairwise relation as seen in strongly monotone learning. Contrasting this similarity, we find substantial differences when studying behaviourally correct monotone learners. Most notably, we show that monotone learners, as opposed to their strongly monotone counterpart, do heavily rely on the order the information is given in, an unusual result for behaviourally correct learners.
\end{abstract}


\section{Introduction}

Algorithmically learning a formal language from a growing but finite amount of its positive information is widely referred to as \emph{inductive inference} or \emph{language learning in the limit}, a branch of (algorithmic) learning theory. In his seminal paper, \citet{Gold67} introduced this setting as follows. A learner $h$ (a computable device) is successively presented all and only the information from a formal language $L$ (a computably enumerable subset of the natural numbers). Such a list of elements of $L$ is called a \emph{text} of $L$. With each new datum, the learner $h$ makes a guess (a description for a computably enumerable set) about which language it is presented using the information shown to it. Once these guesses converge to a single, correct hypothesis explaining the language, the learner successfully \emph{learned} the language $L$ on this text. If $h$ learns $L$ on every text, we say that $h$ \emph{learns} $L$.

This is referred to as \emph{explanatory learning} as the learner, in the limit, provides an explanation of the presented language. We denote this as $\Txt\G\Ex$, where $\Txt$ indicates that the information is given from text, $\G$ stands for \emph{Gold-style} learning, where the learner has \emph{full information} on the elements presented to make its guess, and, lastly, $\Ex$ refers to explanatory learning. Since a single language can be learned by a learner which always guesses this language, we study classes of languages which can be $\Txt\G\Ex$-learned by a single learner and denote the set of all such classes with $[\Txt\G\Ex]$. We refer to this set as \emph{learning power} of $\Txt\G\Ex$-learners.

Inspired from naturally desirable learning behaviour (for example, see \citet{Angluin80} for \emph{conservative} or \citet{OSW82} for \emph{decisive} learning) as well as behaviour witnessed in other sciences, such as psychology (where \citet{PsychUShape} inspired \emph{non-U-shaped} learning \citep{BCMSW08}), various adaptations of $\Txt\G\Ex$-learning have been proposed in the literature. These may affect the amount of information given to the learner, the behaviour the learner may demonstrate as well as the success criterion itself. In this paper, we focus on \emph{monotonic restrictions} as introduced by \citet{Jantke91} and \citet{Wiehagen91}. In its strongest form this natural restriction requires the learner to make monotone guesses, that is, each guess must contain all elements present in previous guesses. This is referred to as \emph{strongly monotone} learning and abbreviated as $\SMon$. We focus on an variant of this criterion, namely \emph{monotone} learning ($\Mon$), where the learner has to be strongly monotone regarding the elements of the target language. 

Monotonic restrictions have been given special attention in the literature. Initially introduced when learning total computable functions, monotone restrictions quickly gained attention when inferring formal languages. \citet{LZ93} intensively studied monotonic learners when given the task of learning indexed families of languages \citep{Angluin80}. Additionally to learning from text, they considered learning from informants, where the learner is provided both positive and negative data to infer its hypotheses from. In both settings, strongly monotone learners are strictly less powerful than their monotone counterpart, which, in turn, are known to be weaker than unrestricted learners. Monotonic learning restrictions formalise the idea of \emph{learning by generalization}, that is, given more and more data the learner produces better and better generalizations of the given information eventually to infer the target language. Looking at this the other way around, \citet{LZK96} consider dual-monotonic learners where learning is achieved by \emph{specialization}. Following this paradigm, the learner, instead of eventually overgeneralizing, produces specializations that fit the target language better and better. One transitions the monotonic learning restrictions to this dual concept by requiring monotonicity on the complement of the hypotheses and target languages. Particularly interesting results show that strongly monotonic learners are strictly stronger regarding learning power than their dual counterpart, while monotone and dual-monotone learners are incomparable to each other.

Surprising results are also obtained when studying monotone learning of indexed families when imposing memory restrictions on the learners. \citet{LZ93} show that, depending on the hypothesis space chosen for the learning task, strongly monotone and monotone learners cope differently with this loss of memory. For example, strongly monotone learners, in any considered case, can be assumed to build their hypotheses solely on the content of the information given. Introduced by \citet{WC80}, such learners are called \emph{set-driven} ($\Sd$). On the other hand, only when the hypothesis space may be chosen freely, monotone learners may neglect the order in which the information was presented to them, that is, they may be assumed \emph{partially set-driven} ($\Psd$), see \citet{BlumBlum75} and \citet{SchRicht84}. Transferring this to learning of arbitrary classes of languages, \citet{KS16} recently studied the behaviour of strongly monotone learners imposed with various memory restrictions. Besides comparing the learning power of \emph{partial} and \emph{total} learners, they also investigate what happens if one requires the restriction to hold \emph{globally}. A key result of \citet{KS16} shows that explanatory strongly monotone learners do not cope well with memory restrictions. In particular, they prove partially set-driven learners to lack full learning power and even more so set-driven learners. However, when requiring the learners to be \emph{behaviourally correct} ($\Bc$), see \citet{CL82} and \citet{OW82}, that is, for correct identification the learner may make finitely many wrong guesses before settling to correct but possibly different conjectures, strongly monotone learners do not require more than the content of the information given and even may exhibit their restriction globally without losing learning power.

\citet{KS16} provide these results within a lucid map, depicted in Figure~\ref{FigSMon}. A picture, which we show to be considerably different when dealing with monotone learners. We provide our results structured as follows. In Section~\ref{Sec:MonEx}, we study the monotonic restrictions of interest when requiring syntactic convergence. In particular, we observe that the overall behaviour of monotone learners resembles the one of strongly monotone learners. This similarity culminates in Theorem~\ref{thm:MonSMon}, where we prove globally monotone learners to be equal to globally strongly monotone ones. We additionally observe that most proof strategies from the strongly monotone case, see \citet{KS16}, can be carried over to fit monotone learners. While these transitions are often non-trivial, they do indicate a deep similarity between these two restrictions. Gathering the results throughout this section, we obtain the map shown in Figure~\ref{FigMonEx} depicting the overall picture of the various discussed monotonic learning restrictions.

\begin{figure}[h]
  \centering
  \begin{adjustbox}{width=0.90\textwidth}
    \includegraphics{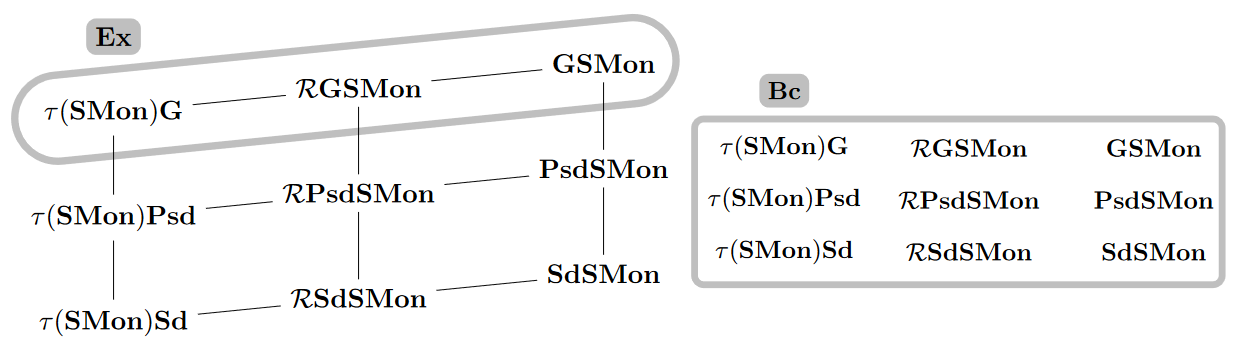}
  \end{adjustbox}
  \caption{Relation of various strongly monotone learning restrictions \citep{KS16}. On the left-hand side we see the explanatory setting ($\Ex$), on the right-hand side the behaviourally correct one ($\Bc$). We omit mentioning $\Txt$ in favour of readability. We write $\totalCp$ when requiring the learner to be total and $\tau(\SMon)$ when requiring the learner to be globally strongly monotone. Black solid lines imply trivial inclusions (bottom-to-top, left-to-right), which we omit drawing in the $\Bc$-case. Furthermore, greyly edged areas illustrate a collapse of the enclosed learning criteria and there are no further collapses.}\label{FigSMon}
\end{figure}

\begin{figure}[h]
  \centering
  \begin{adjustbox}{width=0.70\textwidth}
    \includegraphics{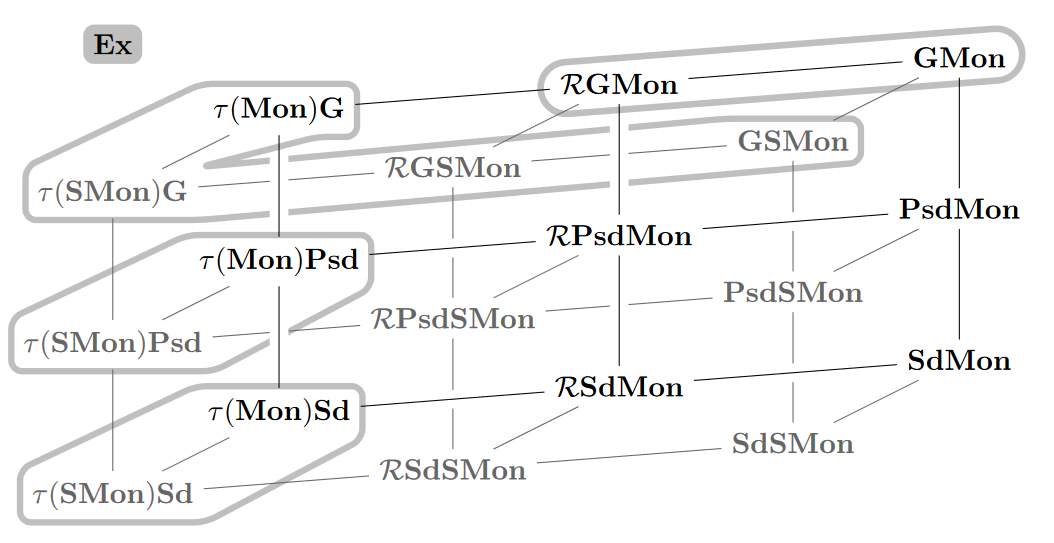}
  \end{adjustbox}
  \caption{Relation of various monotonic learning restrictions in the explanatory case ($\Ex$). We omit mentioning $\Txt$ to favour readability. Furthermore, $\totalCp$ indicates that the considered learners are required to be total and, given a learning restriction $\Lr$, $\tau(\Lr)$ indicates that the restriction $\Lr$ is to hold globally. Solid lines imply trivial inclusions (bottom-to-top, left-to-right). Greyly edged areas illustrate a collapse of the enclosed learning criteria. There are no further collapses.}\label{FigMonEx}
\end{figure}

In Section~\ref{Sec:MonBc}, we transfer this problem to behaviourally correct learners, that is, learners which are required to converge semantically, and discover an unexpected result. In Theorem~\ref{Thm:CoolSep}, we show that Gold-style monotone learners are strictly more powerful than their partially set-driven counterpart. This is particularly surprising as usually behaviourally correct learners cope rather well with such memory restrictions, see for example \citet{KS16} or \citet{DoskocK20}. This marks the most important and surprising insight of this work. In Figure~\ref{FigMonBc}, we collect our findings. Finally, we conclude our work in Section~\ref{Sec:Concl}.

\begin{figure}[h]
  \centering
  \begin{adjustbox}{width=0.45\textwidth}
    \includegraphics{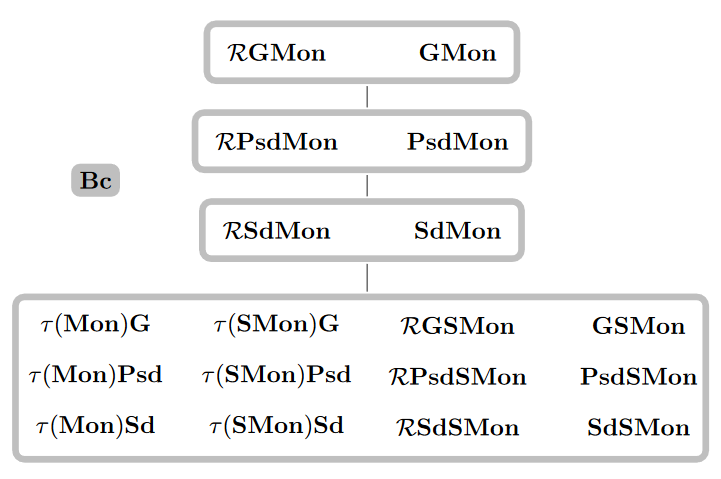}
  \end{adjustbox}
  \caption{Relation of various monotonic learning restrictions in the behaviourally correct case ($\Bc$). To ease readability, we omit mentioning $\Txt$. Additionally, $\totalCp$ indicates that the learner needs to be total, while, given a learning restriction $\Lr$, $\tau(\Lr)$ indicates that $\Lr$ needs to hold globally. Black solid lines imply inclusions (bottom-to-top), while greyly edged areas illustrate a collapse of the enclosed learning criteria.}\label{FigMonBc}
\end{figure}

\section{Language Learning in the Limit}

In this section we introduce notation and preliminary results used throughout this paper. We consider basic computability theory as known, for an overview we refer the reader to \citet{Rogers87}. Furthermore, we follow the system for learning criteria introduced by \citet{Kotzing09}.

\subsection{Preliminaries}

Starting with the mathematical notation, we use $\subsetneq$ and $\subseteq$ to denote the proper subset and subset relation between sets, respectively. We denote with $\N = \{ 0, 1, 2, \dots \}$ the set of all natural numbers. Furthermore, we let $\partialCp$ and $\totalCp$ be the set of all partial and total computable functions $p\colon \N \to \N$. We fix an effective numbering $\{\varphi_e\}_{e \in \N}$ of $\partialCp$ and denote with $W_e = \dom(\varphi_e)$ the $e$-th computably enumerable set. This way, we interpret the natural number $e$ as an \emph{index} or \emph{hypothesis} for the set~$W_e$. We mention important computable functions. Firstly, we fix with $\langle . , . \rangle$ a computable coding function. We use $\pi_1$ and $\pi_2$ to recover the first and second component, respectively. Furthermore, we write $\pad$ for an injective computable function such that, for all $e, k \in \N$, we have $W_e = W_{\pad(e, k)}$. We use $\unpad_1$ and $\unpad_2$ to compute the first and second component, respectively. Note that both functions can be extended iteratively to more coordinates. Lastly, we let $\ind$ compute an index for any given finite set.

We aim to learn \emph{languages}, that is, recursively enumerable sets $L \subseteq \N$. These will be learned by \emph{learners} which are partial computable functions. By $\#$ we denote the \emph{pause symbol} and for any set $S$ we denote $S_\# \coloneqq S \cup \{ \# \}$. Furthermore, a \emph{text} is a total function $T \colon \N \to \N \cup \{ \# \}$, the collection of all texts we denote with $\Txt$. For any text or sequence $T$, we let $\content(T) \coloneqq \range(T) \setminus \{ \# \}$ be the \emph{content} of $T$. A text of a language $L$ is such that $\content(T) = L$, the collection of all texts of $L$ we denote with $\Txt(L)$. For $n \in \N$, we denote by $T[n]$ the initial sequence of $T$ of length $n$, that is, $T[0] \coloneqq \varepsilon$ and $T[n] \coloneqq (T(0), T(1), \dots, T(n-1))$. For a set $S$, we call the text where all elements of $S$ are presented in strictly increasing order (followed by infinitely many pause symbols if $S$ is finite) the \emph{canonical text of $S$}. Furthermore, we call the sequence of all elements of $S$ presented in strictly ascending order the \emph{canonical sequences of $S$}. On finite sequences we use $\subseteq$ to denote the \emph{extension relation} and $\leq$ to denote the order on sequences interpreted as natural numbers. Furthermore, for tuples of finite sets and numbers $(D, t)$ and $(D',t')$, we define the order $\preceq$ such that $(D,t) \preceq (D',t')$ if and only if $t \leq t'$ and there exists a text $T$ such that $D = \content(T[t])$ and $D' = \content(T[t'])$. In addition, given two sequences $\sigma$ and $\tau$ we write $\sigma\concat\tau$ to denote the concatenation of these. Occasionally, we omit writing $\concat$ for readability. 

Next, we formalize learning criteria, following the system proposed by \citet{Kotzing09}. An \emph{interaction operator} $\beta$ takes a learner $h \in \partialCp$ and a text $T \in \Txt$ as argument and outputs a possibly partial function $p$. Intuitively, $\beta$ provides the information for the learner to make its guesses. We consider the interaction operators $\G$ for \emph{Gold-style} or \emph{full-information} learning \citep{Gold67}, $\Psd$ for \emph{partially set-driven} learning \citep{BlumBlum75,SchRicht84} and $\Sd$ for \emph{set-driven} learning \citep{WC80}. Define, for any $i \in \N$,
\begin{align*}
  \G(h, T)(i) &\coloneqq h(T[i]), \\
  \Psd(h, T)(i) &\coloneqq h(\content(T[i]), i), \\
  \Sd(h, T)(i) &\coloneqq h(\content(T[i])).
\end{align*}
Intuitively, Gold-style learners have full information on the elements presented to them. Partially set-driven learners, however, base their guesses on the total amount of elements presented and the content thereof. Lastly, set-driven learners only base their conjectures on the content given to them. Furthermore, for any $\beta$-learner $h$, we write $h^*$ for its starred learner, that is, the $\G$-learner to simulate $h$. For example, if $\beta = \Sd$, then, for any sequence $\sigma$, $h^*(\sigma) = h(\content(\sigma))$. 

When it comes to learning, we can distinguish between various criteria for successful learning. Initially, \citet{Gold67} introduced \emph{explanatory} learning ($\Ex$) as such a learning criterion. Here, a learner is expected to converge to a single, correct hypothesis in order to learn a language. This can be loosened to require the learner to converge semantically, that is, from some point onwards it must output correct hypotheses which may change syntactically \citep{CL82,OW82}. This is referred to as \emph{behaviourally correct} learning ($\Bc$). Formally, a \emph{learning restriction} $\delta$ is a predicate on a total learning sequence $p$, that is, a total function, and a text $T \in \Txt$. For the mentioned criteria we have
\begin{align*}
  \Ex(p, T) &:\Leftrightarrow \exists n_0 \forall n \geq n_0 \colon p(n) = p(n_0) \wedge W_{p(n_0)} = \content(T), \\
  \Bc(p, T) &:\Leftrightarrow \exists n_0 \forall n \geq n_0 \colon W_{p(n)} = \content(T). 
\end{align*}
These success criteria can be expanded in order to model natural learning restrictions or such found in other sciences, say, for example, psychology. In this paper, we focus on natural learning restrictions which incorporate some kind of monotonic behaviour, introduced by \citet{Jantke91} and \citet{Wiehagen91}. \emph{Strongly monotone} learning ($\SMon$) forms the basis. Here, the learner may never discard elements which were once present in its previous hypotheses. This restrictive criterion can be loosened to hold only on the elements of the target language, that is, the learner may never discard elements from the language which it already proposed in previous hypotheses. This is referred to as \emph{monotone} learning ($\Mon$). This is formalized as
\begin{align*}
  \SMon(p, T) &:\Leftrightarrow \forall n, m \colon n \leq m \Rightarrow W_{p(n)} \subseteq W_{p(m)}, \\
  \Mon(p, T) &:\Leftrightarrow \forall n, m \colon n \leq m \Rightarrow W_{p(n)} \cap \content(T) \subseteq W_{p(m)} \cap \content(T).
\end{align*}
Finally, $\T$, the always true predicate, denotes the absence of a restriction.

Now, a \emph{learning criterion} is a tuple $(\alpha, \mathcal{C}, \beta, \delta)$, where $\mathcal{C}$ is a set of admissible learners, typically $\partialCp$ or $\totalCp$, $\beta$ is an interaction operator and $\alpha$ and $\delta$ are learning restrictions. We denote this learning criterion as $\tau(\alpha)\mathcal{C}\Txt\beta\delta$. In the case of $\mathcal{C} = \partialCp$, $\alpha= \T$ or $\delta=\T$ we omit writing the respective symbol. For an admissible learner $h \in \mathcal{C}$ we say that $h$ $\tau(\alpha)\mathcal{C}\Txt\beta\delta$-learns a language $L$ if and only if on arbitrary text $T \in \Txt$ we have $\alpha(\beta(h,T),T)$ and on texts of the target language $T \in \Txt(L)$ we have $\delta(\beta(h,T),T)$. With $\tau(\alpha)\mathcal{C}\Txt\beta\delta(h)$ we denote the class of languages $\tau(\alpha)\mathcal{C}\Txt\beta\delta$-learned by $h$ and the set of all such classes we denote with $[\tau(\alpha)\mathcal{C}\Txt\beta\delta]$.

\subsection{Normal Forms in Inductive Inference}

The introduced learning restrictions all fall into the scope of delayable restrictions. Informally, the hypotheses of a delayable restriction may be postponed arbitrarily but not indefinitely. Formally, we call a learning restriction $\delta$ \emph{delayable} if and only if for all texts $T$ and $T'$ with $\content(T) = \content(T')$, all learning sequences $p$ and all total, unbounded non-decreasing functions $r$, we have that if $\delta(p,T)$ and, for all $n$, $\content(T[r(n)]) \subseteq \content(T'[n])$, then $\delta( p \circ r, T')$. Furthermore, we call a restriction \emph{semantic} if and only if for any learning sequences $p$ and $p'$ and any text $T$, we have that if $\delta(p,T)$ and, for all $n$, $W_{p(n)} = W_{p'(n)}$ implies $\delta(p',T)$. Intuitively, a restriction is semantic if any hypothesis could be replaced by a semantically equivalent one without violating the learning restriction. In particular, one can provide general results when talking about delayable or semantic restrictions. The following theorem holds.
\begin{theorem}[\normalfont{\cite{KP16}; \cite{KSS17}}]
  For all interaction operators $\beta$, all delayable restrictions $\delta$ and all semantic restrictions $\delta'$, we have that \label{Thm:DelTotal}
  \begin{align*}
    [\totalCp\Txt\G\delta] &= [\Txt\G\delta], \\ 
    [\totalCp\Txt\beta\delta'] &= [\Txt\beta\delta'].
  \end{align*}
\end{theorem}
This theorem is particularly useful for us as all mentioned restrictions are delayable and all except for $\Ex$ are semantic.

\section{Studying Monotone Learning Restrictions}\label{Sec:Mon}

In this section we discuss monotone learners under various additional restrictions and also compare them to their strongly monotone counterpart. We split this study into two parts, distinguishing between different convergence criteria. Firstly, we study explanatory such learners, that is, learners which converge syntactically. After that, we investigate learners which converge semantically, that is, behaviourally correct learners.

Before we dive into the respective part, we mention the following result. It is a well-established fact that strongly monotone learners are significantly weaker than their monotone counterpart. In particular, the class $\La = \{ 2\N \} \cup \{ \{ 0, 2, 4, \dots, 2k, 2k +1 \} \mid k \in \N \}$ is learnable by a $\Txt\Sd\Mon\Ex$-learner, however, any $\Txt\G\SMon\Bc$-learner fails to do so. We remark that the separating class can also be learned by a total monotone learner. For completeness, we provide the proof. 

\begin{theorem}
  We have that $[\totalCp\Txt\Sd\Mon\Ex] \setminus [\Txt\G\SMon\Bc] \neq \emptyset$.\label{thm:MonvsSMon}
\end{theorem}

\begin{proof}
  This is a standard proof and we include it for completeness. We adapt the proof from \citet[Thm.~22]{KP16} to also hold for total computable $\Mon$-learners. For $k \in \N$, let $L_{2k+1} = \{ 0, 2, 4, \dots, 2k, 2k+1 \}$ and let $\La = \{ 2\N \} \cup \{ L_{2k+1} \mid k \in \N \}$. Let $e$ be such that $W_e = 2 \N$ and, using the S-m-n Theorem, let $p \in \totalCp$ be such that, for all $k \in \N$, we have $W_{p(2k+1)} = L_{2k+1}$. 

  First, we show that $\La \subseteq \totalCp\Txt\Sd\Mon\Ex(h)$ for the following learner $h$. For all finite sets $D \subseteq \N$ let
  \begin{align*}
    h(D) \coloneqq \begin{cases} e, &\falls D \subseteq 2\N, \\ p(\min(D \setminus 2\N)), &\sonst \end{cases}
  \end{align*}
  First note that $h$ is total computable. Intuitively, while $h$ is presented only even elements, its hypothesis will be $e$, that is, a code for the set of all even numbers. This is the correct behaviour when learning $2\N$. Once it sees an odd element $2k+1$, it changes its hypothesis to a code of $L_{2k+1}$ and never changes its mind again. This is the correct learning behaviour for the language $L_{2k+1}$. Note that this mind change preserves monotonicity. 

  Now assume that there exists a learner $h'$ such that $\La \subseteq \Txt\G\SMon\Bc(h')$. Let $T$ be a text of $2\N$ and let $n_0$ such that, for all $n \geq n_0$, $W_{h'(T[n])} = 2\N$. Let $k$ be such that $\max(\content(T[n_0])) \leq 2k+1$. Then, let $T'$ be a text for $L_{2k+1}$ starting with $T[n_0]$. As $2\N \not\subseteq L_{2k+1}$ we have that $h'$ is either not strongly monotone or does not learn $L_{2k+1}$ from text $T'$ correctly.
\end{proof}

Although monotone learners are considerably more powerful than their strongly monotone counterpart, in Section~\ref{Sec:MonEx}, we observe similarities between explanatory such learners. These similarities are not only reflected by the resulting overall picture, but also by the means of obtaining these results. Thereby, the main difficulty is to reason why the elements used to contradict strongly monotone learning suddenly are part of a learnable language and, thus, also contradict monotone learning. Furthermore, in order to show strong results, all of these adaptations have to be done while maintaining the original learnability by some strongly monotone learner. 

Additionally, these similarities culminate in Theorem~\ref{thm:MonSMon}, where we show globally monotone learners to be equally strong as globally strongly monotone ones. This result also holds true when requiring semantic convergence. However, as monotone learners may discard elements from their guesses, the strategy of keeping all once suggested elements (regardless of the order), presented by \citet{KS16} when studying strongly monotone learners, does not work for such learners. In Theorem~\ref{Thm:CoolSep}, we show that partially set-driven learners are strictly less powerful than their Gold-style counterpart, an unusual result as we discuss in Section~\ref{Sec:MonBc}.

\subsection{Explanatory Monotone Learning} \label{Sec:MonEx}

Here, we investigate monotone learners when requiring syntactic convergence and also compare them to their strongly monotone counterpart. We build our investigations on the work of \citet{KS16}, who conduct a thorough discussion of strongly monotone learners. We show that the general behaviour of both types of learners is alike. This can be seen, firstly, in the resulting overall picture and, secondly, in the way these results are obtained. 

Most notably, our first result is a good indication towards how similar these restrictions are. We show that, when requiring both restrictions to hold globally, the learners are equally powerful. Recall that monotone learners exhibit a strongly monotone behaviour on texts belonging to languages they learn. If they are required to be monotone on \emph{any} possible text, that is, to be globally monotone, they must show strongly monotone behaviour on any text. Thus, they are equally powerful as globally strongly monotone learners. Note that this equality, in fact, holds on the level of the restrictions itself. We provide the following theorem.

\begin{theorem}
  For all restrictions $\delta$ and all interaction operators $\beta$ we have that \label{thm:MonSMon}
  \begin{align*}
    [\tau(\SMon)\Txt\beta\delta] = [\tau(\Mon)\Txt\beta\delta].
  \end{align*}
\end{theorem}

\begin{proof}
  The inclusion $[\tau(\SMon)\Txt\beta\delta] \subseteq [\tau(\Mon)\Txt\beta\delta]$ follows immediately. For the other inclusion, let $h^*$ be a $\tau(\Mon)\Txt\beta\delta$-learner in its starred form. Assume that $h^*$ is not $\tau(\SMon)$. Then, there exists some text $T$, $i<j$ and $x$ such that $x \in W_{h^*(T[i])} \setminus W_{h^*(T[j])}$. Considering the text $T' \coloneqq T[j] ^\frown x ^\frown T(j) ^\frown T(j+1) ^\frown \cdots$, we have
  \begin{align*}
    x \in W_{h^*(T[i])} \cap \content(T') \setminus W_{h^*(T[j])} \cap \content(T').
  \end{align*}
  Thus, $h^*$ is not $\tau(\Mon)$ on text $T'$, a contradiction.
\end{proof}

In particular, this result also implies all separations and equalities known for globally strongly monotone learners to hold for globally monotone learners as well. Most notably, Gold-style globally monotone learners are strictly less powerful than their total counterpart. 

\citet{KP16} show that any Gold-style learner following a delayable restriction can be assumed \emph{total} without loss of learning power. Next, we show that these Gold-style learners are more powerful than their partially set-driven counterpart. In particular, we show that even strongly monotone Gold-style learners are more powerful than any partially set-driven monotone learner. We do so by learning a class of languages on which the learner, in order to discard certain elements, needs to know the order the information appeared in. This, no partially set-driven monotone learner can do. The following result holds.

\begin{theorem}
  We have that $[\Txt\G\SMon\Ex] \setminus [\Txt\Psd\Mon\Ex] \neq \emptyset$.\label{thm:GSMon-PsdMon}
\end{theorem}

\begin{proof}
  We modify the proof of $[\Txt\It\SMon\Ex] \setminus [\Txt\Psd\SMon\Ex] \neq \emptyset$, see \citet[Thm.~4.5]{KS16} to separate $[\Txt\G\SMon\Ex]$ from $[\Txt\Psd\Mon\Ex]$.

  Recall the padding function $\pad \in \totalCp$ and the function $\ind \in \totalCp$ returning the canonical indices of finite sets. Also, fix a pairing function $\langle.,.\rangle$ and, for $i \in \{1,2\}$, let $\pi_i \in \totalCp$ be the function returning the $i$-th component of that pairing function. For finite sequences $\sigma$, we define the auxiliary computable functions 
  \begin{align*}
    w_\sigma &= \begin{cases} 0, &\falls \content(\sigma) \subseteq \{ 0 \}, \\ 1, &\sonst \end{cases} \\
    x_\sigma &= \begin{cases} 0, &\falls |\content(\sigma)| \leq 1, \\ 1, &\sonst \end{cases} \\
    y_\sigma &= \begin{cases} 0, &\falls \forall x \in \content(\sigma) \colon \pi_2(x) = 0, \\ \sigma(i'), &\text{else, with } i' \text{ minimal such that } \pi_2(\sigma(i')) \neq 0. \end{cases} \\
    z_\sigma &= \begin{cases} 0, &\falls \forall x \in \content(\sigma) \colon \pi_2(x) \neq 0, \\ \sigma(i'), &\text{else, with } i' \text{ minimal such that } \pi_2(\sigma(i')) = 0. \end{cases}
  \end{align*}
  Intuitively, $w_\sigma$ changes to one once the first non-zero element appears in $\sigma$, $x_\sigma$ tests whether the sequence's content contains at least two elements, while $y_\sigma$ and $z_\sigma$ search for the first input without and with a zero as second (coding) component, respectively. All these functions are total and change their value at most once. Define the $\G$-learner~$h$ as 
  \begin{align*}
    h(\sigma) = \begin{cases}
      \pad(\ind(\emptyset), 0, 0, 0, 0), &\falls w_\sigma = 0, \\
      \pad(\ind(\{ y_\sigma \}), w_\sigma,  0, y_\sigma, 0), &\sonstfalls w_\sigma \neq 0 \wedge x_\sigma = 0 \wedge y_\sigma \neq 0 \wedge z_\sigma = 0, \\
      \pad(\ind(\{ z_\sigma \}), w_\sigma,  0, 0, z_\sigma), &\sonstfalls w_\sigma \neq 0 \wedge x_\sigma = 0 \wedge y_\sigma = 0 \wedge z_\sigma \neq 0, \\
      \pad(\pi_1(y_\sigma), w_\sigma, x_\sigma, y_\sigma, 0), &\sonstfalls w_\sigma \neq 0 \wedge x_\sigma \neq 0 \wedge y_\sigma \neq 0 \wedge z_\sigma = 0, \\
      \pad(\pi_1( z_\sigma), w_\sigma, x_\sigma, y_\sigma, z_\sigma), &\sonstfalls w_\sigma \neq 0 \wedge x_\sigma \neq 0 \wedge y_\sigma \neq 0 \wedge z_\sigma \neq 0, \\
      \pad(\ind(\emptyset), 0, 0, 0, 0), & \sonst
    \end{cases}
  \end{align*}
  The intuition is the following. Once the learner sees the first non-zero element, it suggests a code for this singleton. This ensures that $h$ learns all singletons except for $\{ 0 \}$. When seeing other elements, as long as no second coding component is zero, $h$ outputs a padded version of the first component of the firstly seen element. This can only be overruled if it sees an element with second coding component zero. Now let $\La = \Txt\G\SMon\Ex(h)$.

  Now assume there exists some (partial) learner $h'$ learning $\La$ in a partially set-driven, monotone way, that is, $\La \subseteq \Txt\Psd\Mon\Ex(h')$. Let $(h')^* = h'(\content(\sigma), |\sigma|)$ denote its starred learner. We will use \textbf{ORT} twice. First, it yields a total computable, strictly monotone increasing function $a \in \totalCp$ such that for all finite sets $D$ we have
  \begin{align*}
    W_{a(D)} = D \cup \{ \langle a(D), 0 \rangle \}.
  \end{align*} 
  As $a$ is strictly monotone it is computably invertible, that is, we can regain the set $D$ from the value $a(D)$. Without loss of generality, $0 \notin \range(a)$. Secondly, there is an index $e$ as well as a computable sequence $(\sigma_i)_{i \in \N}$ of sequences such that the following construction holds. For all $i$, the $\sigma_i$ are recursively defined as
  \begin{align*}
    \sigma_0 &= \varepsilon, \\
    \sigma_{i+1} &= {\sigma_i} ^\frown \begin{cases} \langle e, 2i +1 \rangle^t, &\falls \exists t\colon (h')^*(\sigma_i)\convs \neq (h')^*({\sigma_i} ^\frown \langle e, 2i +1 \rangle^t)\convs, \\
    \langle e, 2i +2 \rangle^t, &\sonstfalls \exists t\colon (h')^*(\sigma_i)\convs \neq (h')^*({\sigma_i} ^\frown \langle e, 2i +2 \rangle^t)\convs. \end{cases} \\
    W_e &= \bigcup_{i \in \N, \sigma_i \convs} \content(\sigma_i).
  \end{align*}
  Intuitively, given $\sigma_i$, we search for an extension on which $h'$ makes a mind change. If none such is found, $\sigma_{j}$, with $i<j$, remain undefined. To enumerate $W_e$, at stage $i$, compute the sequence $\sigma_i$ and, if this computation halts, then enumerate $\content(\sigma_i)$. Note that no element of $W_e$ has $0$ as second component. We now distinguish the following two cases.
  \begin{enumerate}
    \itemin{1. Case:} The language $L \coloneqq W_e$ is infinite. In this case, the sequence $\sigma_i$ is defined for every $i$. We show that $L \in \La$. Let $T \in \Txt(L)$ and let $i'$ such that $x = \langle e, i' \rangle$ is the first element of $L$ to appear in $T$. As we have seen above, when presented the text $T$ the learner $h$ will conjecture $\pad(\ind(\emptyset),0,0,0)$, that is, a code for the empty set, until it sees $x$. Then, it changes its mind to $\pad(\ind(\{ \langle e, i' \rangle \}), \langle e, i' \rangle, 0, 0)$, that is, a code for the singleton $\{ \langle e,i' \rangle \}$. Once the learner is presented another element of the (infinite) language $L$, it returns $\pad(\pi_1(x),\langle e, i' \rangle,1,0)$ and never changes its mind again. Note that $\pi_1(x) = e$. Thus, $h$ converges to a hypothesis of $L$, maintaining strong monotonicity along the way. 

    The learner $h'$ on the other hand cannot identify $L$ from the text $\bigcup_{i \in \N} \sigma_i$ as it makes infinitely many mind changes by definition, a contradiction.

    \itemin{2. Case:} The language $L \coloneqq W_e$ is finite. In this case, from some point onwards, the sequences $\sigma_i$ are not defined any more. Let $\sigma_k$ be the last defined such sequence. In particular, for all $t \in \N$ we have 
    \begin{align}
      (h')^*({\sigma_k} ^\frown {\langle e,2k+1 \rangle^t}) = (h')^*(\sigma_k) = (h')^*({\sigma_k} ^\frown {\langle e,2k+2 \rangle^t}). \label{Prop:sigma_k}
    \end{align}
    We show that $h$ learns $L$ and
    \begin{align*}
      L_1 &\coloneqq L \cup \{ \langle e, 2k+1 \rangle, \langle a(L \cup \{ \langle e, 2k+1 \rangle \} ), 0 \rangle \} \text{ and} \\
      L_2 &\coloneqq L \cup \{ \langle e, 2k+2 \rangle, \langle a(L \cup \{ \langle e, 2k+2 \rangle \} ), 0 \rangle \}.
    \end{align*}
    Note that $\langle e, 2k+1 \rangle \notin L_2$ and $\langle e, 2k+2 \rangle \notin L_1$ and both languages have at least two elements. First, we show that $h$ learns $L$. Let $T \in \Txt(L)$. As long as no element is presented, $h$ outputs a hypothesis for the empty set. Once the first element is presented, it changes its mind to a code of that singleton. As soon as a second element is presented (if ever), it changes its mind to $e$, which is correct. We proceed by showing that $h$ also learns $L_1$ and $L_2$. Let $T \in \Txt(L_1)$. To ease notation, let $D' \coloneqq L \cup \{ \langle e, 2k+1 \rangle \}$. Let $x$ again be the first element to appear in $T$. 
    \begin{enumerate}
      \itemin{2.1. Case:} $x = \langle e, i' \rangle$. Again, $h$ proposes $\pad(\ind(\emptyset),0,0,0)$, that is, (a code of) the empty set, until it sees $x$. Then, it changes its mind to $\pad(\ind(\{ \langle e, i' \rangle \}), \langle e, i' \rangle, 0, 0)$, that is, (a code of) the singleton $\{ \langle e,i' \rangle \}$. It may temporarily propose $\pad(\pi_1(x),\langle e, i' \rangle,1,0)$, where $W_{\pad(\pi_1(x),\langle e, i' \rangle,1,0)} = W_e = L$, and then finally switch to the correct hypothesis $\pad(\pi_1(\langle a(D'), 0 \rangle), \langle e, i' \rangle, 1, \langle a(D'), 0 \rangle)$, that is, the superset 
      \begin{align*}
        W_{\pad(\pi_1(\langle a(D'), 0 \rangle), \langle e, i' \rangle, 1, \langle a(D'), 0 \rangle)} = W_{a(D')} = D' \cup \{ \langle a(D'), 0 \rangle \} = L_1,
      \end{align*}
      once it sees $\langle a(D'), 0 \rangle$ for the first time. 
      \itemin{2.2. Case:} $x = \langle a(D'), 0 \rangle$. In this case, $h$ proposes $\pad(\ind(\emptyset),0,0,0)$, that is, (a code of) the empty set, until it sees $x$. Upon seeing this new element, it changes its mind to $\pad(\ind(\{ \langle a(D'), 0 \rangle \}), \langle a(D'), 0 \rangle, 0, 0)$, that is, (a code of) the singleton $\{ \langle a(D'), 0 \rangle \}$. Once it sees another element, it skips to $\pad(\pi_1(\langle a(D'), 0 \rangle), \langle a(D'), 0 \rangle, 1, \langle a(D'), 0 \rangle)$, that is, a hypothesis for the superset 
      \begin{align*}
        W_{\pad(\pi_1(\langle a(D'), 0 \rangle), \langle a(D'), 0 \rangle, 1, \langle a(D'), 0 \rangle)} = W_{a(D')} = D' \cup \{ \langle a(D'), 0 \rangle \} = L_1.
      \end{align*}
    \end{enumerate}
    In both cases the learner $h$ converges to the correct hypothesis, while maintaining strong monotonicity. The reasoning for $L_2$ is analogous.

    We proceed by showing that $h'$ cannot accomplish the same learning tasks. First, note that, for every index $i > 0$, the set $\{ \langle e,i \rangle \}$ is in $\La$ as $h$ learns all singletons except $\{ 0 \}$. Thus, by assumption we get $\{ \langle e,i \rangle \} \in \Txt\Psd\Mon\Ex(h')$. Therefore, there exists a (computable) function $t$ such that for every $i>0$
    \begin{align*}
      \langle e,i \rangle \in W_{(h')^*(\langle e,i \rangle^{t(i)})}.
    \end{align*}
    The learner $h'$ needs to be both defined and monotone on sequences ${\langle e,2k+1 \rangle^{t(2k+1)}}^\frown \sigma_k$ and ${\langle e,2k+2 \rangle^{t(2k+2)}}^\frown \sigma_k$ as they are initial sequences of texts for $L_1$ and $L_2$. Thus, we get 
    \begin{align*}
      \langle e,2k+1 \rangle &\in W_{(h')^*({\langle e,2k+1 \rangle^{t(2k+1)}}^\frown \sigma_k)} \text{ and} \\
      \langle e,2k+2 \rangle &\in W_{(h')^*({\langle e,2k+2 \rangle^{t(2k+2)}}^\frown \sigma_k)}.
    \end{align*}
    By Property~\eqref{Prop:sigma_k} of the sequence $\sigma_k$ and the partially set-drivenness of $h'$ we have, for all $j \geq \max(t(2k+1), t(2k+2))$,
    \begin{align*}
      (h')^*({\langle e,2k+1 \rangle^j} ^\frown \sigma_k) = (h')^*(\sigma_k) = (h')^*({\langle e,2k+2 \rangle^j}^\frown \sigma_k).
    \end{align*}
    Now $\sigma_k$ is an initial sequence of elements in $L = W_e$, but also 
    \begin{align*}
      \left\{ \langle e, 2k+1 \rangle, \langle e, 2k+2 \rangle \right\} \subseteq W_{(h')^*(\sigma_k)}.
    \end{align*}
    As $h'$ is monotone, and $\sigma_k$ is also an initial sequence of a text of $L_1$ and $L_2$, we get for every $\tau \in L^*$ that
    \begin{align*}
      \left\{ \langle e, 2k+1 \rangle, \langle e, 2k+2 \rangle \right\} \subseteq W_{(h')^*(\sigma_k \tau)}.
    \end{align*}
    Thus, $(h')^*(\sigma_k \tau)$ is not a correct hypothesis for $L$. Consequently, $h'$ cannot learn $L$, a contradiction. \qedhere
  \end{enumerate} \noqed 
\end{proof}

Next, we show that a partial learner, even sustaining a severe memory restriction and expected to be strongly monotone, is still more powerful than any total monotone, partially set-driven learner. In order to construct a separating class of languages, the trick is that the total learner \emph{must} make a guess, a decision which a partial learner can await and, thus, learn more languages. The following result holds.

\begin{theorem}
  We have that $[\Txt\Sd\SMon\Ex] \setminus [\totalCp\Txt\Psd\Mon\Ex] \neq \emptyset$.
\end{theorem}

\begin{proof}
  We adapt the proof of \citet[Thm.~4.1]{KS16} as follows. Let $h \in \partialCp$ be the following learner. With $p_0$ being such that $W_{p_0} = \emptyset$, let for each finite set $D \subseteq \N$
  \begin{align*}
    h(D) = \begin{cases}
      \pad(p_0,0), &\falls D = \emptyset, \\
      \pad(\ind(D),0), & \sonstfalls |D| = 1, \\
      \divs, &\sonstfalls \exists x \in D\colon \varphi_x(0)\divs \vee \unpad_2(\varphi_x(0)) \notin \{1,2\}, \\
      e, &\sonstfalls \forall x \in D\colon \unpad_1(\varphi_x(0)) = e, \\
      e', &\sonstfalls \big( \exists y \forall x \in D \colon \unpad_2(\varphi_x(0)) = 1 \Rightarrow \varphi_x(0) = y \big) \wedge \\
      &\hspace*{-0.175em}\phantom{else, if} \big( \forall x \in D\colon \unpad_2(\varphi_x(0)) = 2 \Rightarrow \varphi_x(0) = e'\big) , \\
      \divs, &\sonst
    \end{cases}
  \end{align*}
  The intuition is the following. While no elements are presented, $h$ conjectures the empty set. Once, a single element is presented, $h$ suggests that singleton. Thus, $h$ learns all singletons. Given more elements, $h$ either outputs the first coordinate of the elements (if they all coincide), or another code if there are different second coordinates.

  Let $\La = \Txt\Sd\SMon\Ex(h)$. Assume there exists a $\totalCp\Txt\Psd\Mon\Ex$-learner $h'$ which learns $\La$, that is, $\La \subseteq \totalCp\Txt\Psd\Mon\Ex(h')$. Since $h$ learns all singletons, so does $h'$. Thus, there is a total, strictly monotone function $t \in \totalCp$ such that $t(0) > 0$ and for each $x$
  \begin{align}
    x \in W_{h'(\{x\},t(x))}. \label{MonSingleton}
  \end{align}
  With \ORT, we get a total recursive predicate $P \in \totalCp$, a strictly monotone $a \in \totalCp$ and indices $e, e' \in \N$ such that for all $i \in \N$
  \begin{align*}
    P(i) &\Leftrightarrow h'(\content(a[i]), t(a(i))+i) \neq h'(\content(a[i+1]), t(a(i))+i+1), \\
    W_e &= \{ a(i) \mid \forall j \leq i \colon P(j) \}, \\
    W_{e'} &= \{ a(i) \mid \forall j < i \colon P(j) \}, \\
    \varphi_{a(i)}(0) &= \begin{cases} \pad(e,1), &\falls P(i), \\ \pad(e',2), &\sonst \end{cases}
  \end{align*}

  We show that $W_e$ and $W_{e'}$ are in $\La$. 
  \begin{enumerate}
    \itemin{1. Case:} $W_e$ is infinite. This means for all $i$ we have $P(i)$. Thus, $W_e = W_{e'}$. Thus, it suffices to show $W_e \in \La$. Let $T \in \Txt(W_e)$. For $n>0$, let $D_n \coloneqq \content(T[n])$. As long as $D_n = \emptyset$, we have $h(D_n) = \pad(p_0, 0)$, i.e.\ a code for the empty set. When $|D_n| = 1$, we have $h(D_n) = \pad(\ind(D_n),0)$, a code for the singleton $D_n$. Once $D_n$ contains more than one element, $h(D_n)$ starts unpadding. As, for all $i$, $\varphi_{a(i)}(0) = \pad(e,1)$, we have $\unpad_1(\{\varphi_x(0) \mid x \in D_n \}) = \{ e \}$. Thus, $h$ is strongly monotone and will output $e$ correctly.
    \itemin{2. Case:} $W_e$ is finite. Let $k$ be such that $W_e = \{ a(j) \mid j < k \}$ and $W_{e'} = \{ a(j) \mid j < k+1 \}$. Again, as long as no elements or only one element is shown, $h$ will output a code for the empty, respectively singleton set. As $W_e \subseteq W_{e'}$ and $\unpad_1(\{ \varphi_x(0) \mid x \in W_e \}) = \{ e \}$, $h$ will output $e$ as long as it sees only elements from $W_e$. Once it sees $a(k) \in W_{e'}$, it correctly changes its mind to $e'$. This maintains strong monotonicity, and is the correct behaviour. 
  \end{enumerate}
  Thus, $W_e, W_{e'} \in \La$. We show that $h'$ cannot learn both simultaneously.
  \begin{enumerate}
    \itemin{1. Case:} $W_e$ is infinite. On the following text of $W_e$
    \begin{align*}
      a(0)^{t(0)} a(1)^{t(1)+1} a(2)^{t(2)+2} \dots,
    \end{align*}
    learner $h'$ makes infinitely many mind changes. Thus, it cannot learn $W_e$, a contradiction.
    \itemin{2. Case:} $W_e$ is finite. Let $k$ be minimal such that $\neg P(k)$, and thus $W_e = \content(a[k])$ and $W_{e'} = \content(a[k+1])$. By Condition~\eqref{MonSingleton} and monotonicity of $h'$ on $W_{e'}$ we have 
    \begin{align*}
      a(k) \in W_{h'(\content(a[k+1]), t(a(k))+k+1)},
    \end{align*} 
    as ${a(k)^{t(a(k))}}^\frown a[k]$ is a sequence of elements in $W_{e'}$ and $a(k) \in W_{e'}$. Since $\neg P(k)$, we get that $h'(\content(a[k]), t(a(k))+k) = h'(\content(a[k+1]), t(a(k)) + k +1)$ and thus
    \begin{align*}
      a(k) \in W_{h'(\content(a[k]),t(a(k))+k)}.
    \end{align*}
    For each $t \geq t(a(k))+k$, we have that $(\content(a[k]),t)$ is an initial sequence for some text of $W_{e'}$, and thus, by monotonicity of $h'$ we get
    \begin{align*}
      a(k) \in W_{h'(\content(a[k]),t)}.
    \end{align*}
    As $a(k) \notin W_e = \content(a[k])$, $h'$ cannot identify $W_e$, a contradiction. \qedhere
  \end{enumerate}\noqed
\end{proof}

Lastly, it remains to be shown that globally strongly monotone, partially set-driven learners are more powerful than their monotone, set-driven counterpart. A separation from strongly monotone set-driven learners has already been shown by \citet{KS16}. We observe that, with a slight adaptation of their proof, one can obtain an even stronger result. We show that globally strongly monotone, partially set-driven learners outperform \emph{unrestricted} set-driven learners. This powerful result shows the immense weakness of set-driven learners which results from a lack of ``learning time'' and strengthens the finding of \citet{Fulk90}, who separated explanatory $\G$-learners from set-driven $\Bc$-learners, even more. He stated that ``\emph{[$\dots$] it is worthwhile to have some time to think over one’s experiences; merely to experience something is not always sufficient to understand it.}'' We provide the result.

\begin{theorem}
  We have that $[\tau(\SMon)\Txt\Psd\Ex] \setminus [\Txt\Sd\Ex] \neq \emptyset$.
\end{theorem}

\begin{proof}
  We adapt the proof of $[\tau(\SMon)\Txt\Psd\Ex] \setminus [\Txt\Sd\SMon\Ex] \neq \emptyset$, see \citet[Thm.~4.8]{KS16}. There, one can see that replacing $[\Txt\Sd\SMon\Ex]$ by $[\Txt\Sd\Ex]$ still works out. We include the proof for completeness reasons.

  Let $p_0$ be a code for the empty set and $p_2$ one for the set $\N$. Furthermore, let $\join \in \totalCp$ be a total computable function such that, for all $e \in \N$ and all finite sets $D \subseteq \N$, we have $W_{\join(e,D)} = W_e \cup D$. We consider the following learner $h \in \totalCp$. For any number $t \in \N$ and any finite $D\subseteq \N$, we let
  \begin{align*}
    h(D,t) = \begin{cases} 
      p_0, &\falls D = \emptyset, \\
      p_2, &\sonstfalls |\unpad_1(D)| > 1 \text{ or } |\unpad_2(D)| > 1, \\
      e, &\sonstfalls \exists p\colon \forall x \in D \exists i \colon x = \langle e,p, i \rangle \wedge \varphi_p(0) \text{ does not halt in } t \text{ steps}, \\
      \join(e,D), &\sonst
    \end{cases}
  \end{align*}
  First, we show that $h$ is strongly monotone on arbitrary texts. As long as no information is presented, that is, $D = \emptyset$, it outputs a code for the empty set. As long as all data is of the form, for some fixed $e,p \in \N$ and various $i \in \N$, $\langle e,p,i \rangle$ and the program $\varphi_p(0)$ does not halt in $t$ steps, the set $W_e$ is proposed. Once the halting is witnessed, if ever, $h$ changes its mind to some code of the superset $W_e \cup D$. If multiple first or second coordinates occur, $h$ conjectures $\N$ as its final guess. Now, let $\La = \tau(\SMon)\Txt\Psd\Ex(h')$.

  By way of contradiction, assume there exists some learner $h'$ such that $\La \subseteq \Txt\Sd\Bc(h')$. As $\N \in \La$, the learner $h'$ needs to be total. With \ORT, we get $e, p \in \N$ such that, using $\llangle e, p, j \rrangle \coloneqq \{ \langle e, p, i \rangle \mid i \leq j \}$ as an abbreviation,
  \begin{align*}
    W_e &= \{ \langle e,p,i \rangle \mid \forall j \leq i\colon  h'(\llangle e, p, j\rrangle) \neq h'(\llangle e, p, j+1 \rrangle) \}, \\
    \varphi_p(0) &= \begin{cases}
      1, &\falls \exists i \colon h'(\llangle e, p, i\rrangle) = h'(\llangle e, p, i+1 \rrangle) \}, \\
      \divs, & \sonst
    \end{cases}
  \end{align*}
  We show that there are languages $h$ learns, which $h'$ cannot learn. To that end, we make the following case distinction.
  \begin{enumerate}
    \itemin{1. Case:} The set $W_e$ is infinite. Then, $W_e = \{ \langle e,p,i \rangle \mid i \in \N \}$. In this case, $\varphi_p(0)$ never halts, so $h$ given any information about $W_e$ always outputs $e$, the correct code. On the other hand, $h'$ cannot learn $W_e$ from text $T \colon i \mapsto \langle e, p, i \rangle$ as it makes infinitely many mind changes, a contradiction.
    \itemin{2. Case:} The set $W_e$ is finite. Then, there exists $k$ such that $W_e = \llangle e, p, k \rrangle$. As $\langle e,p,k+1 \rangle$ is not in $W_e$, we have $h'(W_e) = h'(W_e \cup \{ \langle e,p, k+1 \rangle \})$. \newline
    In particular, there is a point $t$ where $\varphi_p(0)$ converges after $t$ steps. This implies that both finite languages $L = W_e$ and $L' = W_e \cup \{ \langle e,p, k+1 \rangle \}$ are in $\La$, as learner $h$ converges to the correct hypotheses, that is, $\join(e,L)$ and $\join(e,L')$ respectively. On the other hand, the learner $h'$ does not distinguish between $L$ and $L'$ as $h'(L) = h'(L')$, a contradiction. \qedhere
  \end{enumerate}\noqed
\end{proof}

Altogether, we gathered the necessary results to expand the explanatory strongly monotone map presented by \citet{KS16}, see Figure~\ref{FigSMon}, to also contain monotone learners. Our results are depicted in Figure~\ref{FigMonEx}.

\subsection{Behaviourally Correct Monotone Learning}\label{Sec:MonBc}

In this section we consider an analogous question: how do monotone and strongly monotone learners interact when requiring semantic convergence? By Theorem~\ref{thm:MonSMon} and the findings of \citet{KS16}, we already have that globally monotone set-driven (and even Gold-style) learners are as powerful as strongly monotone Gold-style learners. This is already a significant difference to the results obtained in the previous section. Most notably, this implies that an analogous result to Theorem~\ref{thm:GSMon-PsdMon}, where Gold-style $\SMon$-learners are shown to be more powerful than partially set-driven $\Mon$-learners, cannot hold true in the case of semantic convergence. The question arises, whether Gold-style $\Mon$-learners even can be separated from partially set-driven $\Mon$-learners in this case? Studies of various other restrictions, conducted by \citet{KS16} and \citet{DoskocK20}, show that behaviourally correct partially set-driven learners are as powerful as their respective Gold-style counterpart.

Surprisingly, for monotone behaviourally correct learners, such a equality does \emph{not} hold true, as we show with the next result. The idea is to construct a class of languages where the learner must keep track of the order the elements were presented in, in order to safely discard them at a later point in learning-time. To obtain this result, we apply the technique of self-learning classes, presented by \citet{CK16}, using the Operator Recursion Theorem, see \citet{Case74}. 

\begin{theorem}
  We have that $[\Txt\G\Mon\Ex] \setminus [\Txt\Psd\Mon\Bc] \neq \emptyset$. \label{Thm:CoolSep}
\end{theorem}

\begin{proof}
  We provide a class witnessing the separation using self-learning classes, as presented in \citet[Thm.~3.6]{CK16}. Consider the learner which for a finite sequence $\sigma$ is defined as 
  $$h(\sigma) = \begin{cases}
    \ind(\emptyset), &\falls \content(\sigma) = \emptyset, \\
    \varphi_{\max(\content(\sigma))}(\sigma), &\sonst
  \end{cases}$$
  Let $\La = \Txt\G\Mon\Ex(h)$. Assume there exists a $\Txt\Psd\Mon\Bc$-learner $h'$ which learns $\La$, that is, $\La \subseteq \Txt\Psd\Mon\Bc(h')$. By the Operator Recursion Theorem (\ORT), see \citet{Case74}, there exist a family of strictly monotone, total computable functions $(a_j)_{j \in \N}$ with pairwise disjoint range, a total computable function $f \in \totalCp$, an index $e \in \N$ and two families of indices $(e_j)_{j \in \N}, (\hat{e}_k)_{k \in \N}$ such that for all finite sequences~$\sigma$, where $\first(\sigma)$ is the first non-pause element in the sequence $\sigma$, we have
  \begin{align*}
    \varphi_{a_j(i)}(\sigma) &= \begin{cases}
    e_j, &\falls \content(\sigma) \subseteq \range(a_j), \\
    \hat{e}_k, &\sonstfalls \exists k \colon a_k(f(k)) \in \content(\sigma) \ \vee \\
    & \exists k \colon \first(\sigma) \in \range(a_k) \wedge \max \{ j \mid \content(\sigma) \cap \range(a_j) \neq \emptyset \} = k, \\
    e, &\sonst
    \end{cases} \\
    f(j) &= \text{first $i$ found such that } a_j(i) \in W_{h'(\content(a_j[i]), i)}, \\
    W_{e_j} &= \range(a_j), \\
    W_{\hat{e}_k} &= \bigcup_{j' \leq k} \content(a_{j'}[f(j')]) \cup \{ a_k(f(k)) \}, \\
    W_e &= \bigcup_j \content(a_j[f(j)]).
  \end{align*}
  
  Let $\La' = \{ W_{e_j} \mid j \in \N\} \cup \{ W_{\hat{e}_k} \mid k >0\} \cup \{ W_e \}$. A depiction of the class $\La'$ can be seen in Figure~\ref{Fig:HardestProof}. We show that $\La'$ can be learned by $h$, but not by $h'$, that is, $\La' \subseteq \La = \Txt\G\Mon\Ex(h)$ but also $\La' \not \subseteq \Txt\Psd\Mon\Bc(h')$. The intuition is the following. For some $j$, as long as only elements from $W_{e_j}$ are presented, $h$ will suggest $e_j$ as its hypothesis. Thus, $h'$ needs to learn $W_{e_j}$ as well and eventually overgeneralize, that is, at some point $i$ we have $\content(a_j[i]) \subsetneq W_{h'(\content(a_j[i]), i)}$. The function $f(j)$ finds such $i$. Once the overgeneralization happenes, we proceed by showing, for $j' \neq j$, elements from $\range(a_{j'})$. Knowing the order in which the elements were presented, the learner $h$ now either keeps or discards the element $a_j(f(j))$ in its next hypothesis depending whether $j'<j$ or $j<j'$, respectively. If $j' < j$, $h$ needs to keep $a_j(f(j))$ in its hypothesis as it still may be presented the set $W_{\hat{e}_j}$. Otherwise, it suggests the set $W_e$, only changing its mind if it sees, for appropriate $i \in \N$, an element of the form $a_i(f(i))$. Then, $h$ is certain to be presented $W_{\hat{e}_i}$. So the full-information learner $h$ can deal with this new information and preserve monotonicity, while $h'$ cannot, as it does not know which information came first. 

  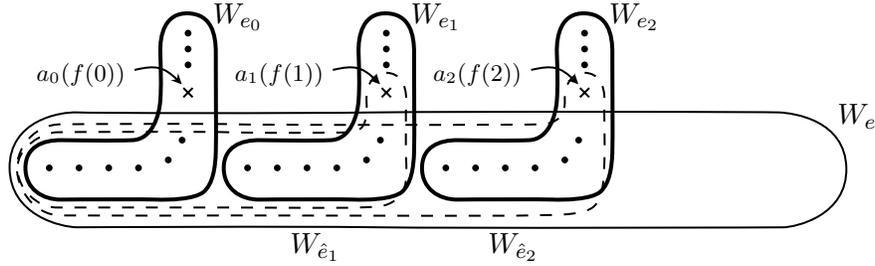
\begin{figure}[h]
    \centering
    \def\CrossSize{2.5}
    \def\DotSize{0.9}
    \tikzset{every picture/.style={line width=0.75pt}}
    \tikzset{cross/.style={cross out, draw=black, minimum size=2*(#1-\pgflinewidth), inner sep=0pt, outer sep=0pt}, cross/.default={1pt}}       

    \begin{tikzpicture}[x=0.75pt,y=0.75pt,yscale=-1,xscale=1]

      \draw  [line width=1.5]  (150,76) .. controls (161.75,76.34) and (164.04,86.42) .. (164,96) .. controls (163.96,105.58) and (164.25,137.99) .. (164,146) .. controls (163.75,154.01) and (162.54,161.92) .. (158,166) .. controls (153.46,170.08) and (139.27,170.21) .. (134,170) .. controls (128.73,169.79) and (95.51,169.89) .. (86,170) .. controls (76.49,170.11) and (67.95,164.57) .. (68,154) .. controls (68.05,143.43) and (76.54,139.92) .. (86,140) .. controls (95.46,140.08) and (112.04,140.42) .. (122,140) .. controls (131.96,139.58) and (135.46,127.58) .. (136,120) .. controls (136.54,112.42) and (136.29,105.08) .. (136,96) .. controls (135.71,86.92) and (138.25,75.66) .. (150,76) -- cycle ;

      \draw   (448,126) .. controls (455.61,126.31) and (481.42,129.97) .. (482,154) .. controls (482.58,178.03) and (456.67,183.94) .. (450,184) .. controls (443.33,184.06) and (289.46,184.56) .. (244.5,183.97) .. controls (199.54,183.39) and (101.83,183.97) .. (94,183.97) .. controls (86.17,183.97) and (60.17,178.03) .. (60,154) .. controls (59.83,129.97) and (85.5,125.64) .. (94,126) .. controls (102.5,126.36) and (198.46,126.58) .. (244.5,126) .. controls (290.54,125.42) and (440.39,125.69) .. (448,126) -- cycle ;

      \draw  [dash pattern={on 4.5pt off 4.5pt}] (250,106) .. controls (256.73,106.04) and (260.21,111.92) .. (260,120) .. controls (259.79,128.08) and (259.75,142.91) .. (260,150) .. controls (260.25,157.09) and (260.82,165.94) .. (255,170) .. controls (249.18,174.06) and (245.79,173.58) .. (236,174) .. controls (226.21,174.42) and (174.6,173.96) .. (170,174) .. controls (165.4,174.04) and (94.21,173.92) .. (84,174) .. controls (73.79,174.08) and (63.95,164.57) .. (64,154) .. controls (64.05,143.43) and (76.42,136.17) .. (84,136) .. controls (91.58,135.83) and (160.29,135.58) .. (170,136) .. controls (179.71,136.42) and (223.71,136.42) .. (232,136) .. controls (240.29,135.58) and (240.29,127.58) .. (240,120) .. controls (239.71,112.42) and (243.27,105.96) .. (250,106) -- cycle ;
 
      \draw  [dash pattern={on 4.5pt off 4.5pt}] (350,106) .. controls (356.73,106.04) and (360.21,111.92) .. (360,120) .. controls (359.79,128.08) and (360.29,143.58) .. (360,150) .. controls (359.71,156.42) and (360.27,169.49) .. (354,174) .. controls (347.73,178.51) and (331.79,177.58) .. (322,178) .. controls (312.21,178.42) and (180.32,177.96) .. (175.71,178) .. controls (171.11,178.04) and (99.13,178.17) .. (85.71,178) .. controls (72.3,177.83) and (63.66,164.57) .. (63.71,154) .. controls (63.77,143.43) and (72.63,132.17) .. (85.71,132) .. controls (98.8,131.83) and (164.01,131.58) .. (173.71,132) .. controls (183.42,132.42) and (319.71,132.42) .. (328,132) .. controls (336.29,131.58) and (340.29,127.58) .. (340,120) .. controls (339.71,112.42) and (343.27,105.96) .. (350,106) -- cycle ;

      \draw  [line width=1.5]  (250,76) .. controls (261.75,76.34) and (264.04,86.42) .. (264,96) .. controls (263.96,105.58) and (264.25,137.99) .. (264,146) .. controls (263.75,154.01) and (262.54,161.92) .. (258,166) .. controls (253.46,170.08) and (239.27,170.21) .. (234,170) .. controls (228.73,169.79) and (195.51,169.89) .. (186,170) .. controls (176.49,170.11) and (167.95,164.57) .. (168,154) .. controls (168.05,143.43) and (176.54,139.92) .. (186,140) .. controls (195.46,140.08) and (212.04,140.42) .. (222,140) .. controls (231.96,139.58) and (235.46,127.58) .. (236,120) .. controls (236.54,112.42) and (236.29,105.08) .. (236,96) .. controls (235.71,86.92) and (238.25,75.66) .. (250,76) -- cycle ;

      \draw  [line width=1.5]  (350,76) .. controls (361.75,76.34) and (364.04,86.42) .. (364,96) .. controls (363.96,105.58) and (364.25,137.99) .. (364,146) .. controls (363.75,154.01) and (362.54,161.92) .. (358,166) .. controls (353.46,170.08) and (339.27,170.21) .. (334,170) .. controls (328.73,169.79) and (295.51,169.89) .. (286,170) .. controls (276.49,170.11) and (267.95,164.57) .. (268,154) .. controls (268.05,143.43) and (276.54,139.92) .. (286,140) .. controls (295.46,140.08) and (312.04,140.42) .. (322,140) .. controls (331.96,139.58) and (335.46,127.58) .. (336,120) .. controls (336.54,112.42) and (336.29,105.08) .. (336,96) .. controls (335.71,86.92) and (338.25,75.66) .. (350,76) -- cycle ;

      \draw (161,69.4) node [anchor=north west][inner sep=0.75pt]    {$W_{e_{0}}$};
      \draw (261,69.4) node [anchor=north west][inner sep=0.75pt]    {$W_{e_{1}}$};
      \draw (361,69.4) node [anchor=north west][inner sep=0.75pt]    {$W_{e_{2}}$};
      \draw (476,120.4) node [anchor=north west][inner sep=0.75pt]    {$W_{e}$};
      \draw (200,185.4) node [anchor=north west][inner sep=0.75pt]    {$W_{\hat{e}_{1}}$};
      \draw (300,185.4) node [anchor=north west][inner sep=0.75pt]    {$W_{\hat{e}_{2}}$};

      \node at (150,116) [cross=\CrossSize pt] {};
      \node at (250,116) [cross=\CrossSize pt] {};
      \node at (350,116) [cross=\CrossSize pt] {};

      \node at (80, 154) [circle,fill,inner sep=\DotSize pt]{};
      \node at (95, 154) [circle,fill,inner sep=\DotSize pt]{};
      \node at (110, 154) [circle,fill,inner sep=\DotSize pt]{};
      \node at (125, 154) [circle,fill,inner sep=\DotSize pt]{};
      \node at (140, 150) [circle,fill,inner sep=\DotSize pt]{};
      \node at (147, 140) [circle,fill,inner sep=\DotSize pt]{};

      \node at (150, 102) [circle,fill,inner sep=\DotSize pt]{};
      \node at (150, 94) [circle,fill,inner sep=\DotSize pt]{};
      \node at (150, 86) [circle,fill,inner sep=\DotSize pt]{};

      \node at (180, 154) [circle,fill,inner sep=\DotSize pt]{};
      \node at (195, 154) [circle,fill,inner sep=\DotSize pt]{};
      \node at (210, 154) [circle,fill,inner sep=\DotSize pt]{};
      \node at (225, 154) [circle,fill,inner sep=\DotSize pt]{};
      \node at (240, 150) [circle,fill,inner sep=\DotSize pt]{};
      \node at (247, 140) [circle,fill,inner sep=\DotSize pt]{};

      \node at (250, 102) [circle,fill,inner sep=\DotSize pt]{};
      \node at (250, 94) [circle,fill,inner sep=\DotSize pt]{};
      \node at (250, 86) [circle,fill,inner sep=\DotSize pt]{};

      \node at (280, 154) [circle,fill,inner sep=\DotSize pt]{};
      \node at (295, 154) [circle,fill,inner sep=\DotSize pt]{};
      \node at (310, 154) [circle,fill,inner sep=\DotSize pt]{};
      \node at (325, 154) [circle,fill,inner sep=\DotSize pt]{};
      \node at (340, 150) [circle,fill,inner sep=\DotSize pt]{};
      \node at (347, 140) [circle,fill,inner sep=\DotSize pt]{};

      \node at (350, 102) [circle,fill,inner sep=\DotSize pt]{};
      \node at (350, 94) [circle,fill,inner sep=\DotSize pt]{};
      \node at (350, 86) [circle,fill,inner sep=\DotSize pt]{};

      \draw    (123,104.57) .. controls (133.98,101.96) and (137.78,102.93) .. (144.81,109.77) ;
      \draw [shift={(146.86,111.83)}, rotate = 225.87] [fill={rgb, 255:red, 0; green, 0; blue, 0 }  ][line width=0.08]  [draw opacity=0] (5.36,-2.57) -- (0,0) -- (5.36,2.57) -- (3.56,0) -- cycle    ;
      \node at (96, 107) {\footnotesize{$a_0(f(0))$}};

      \draw    (223,104.57) .. controls (233.98,101.96) and (237.78,102.93) .. (244.81,109.77) ;
      \draw [shift={(246.86,111.83)}, rotate = 225.87] [fill={rgb, 255:red, 0; green, 0; blue, 0 }  ][line width=0.08]  [draw opacity=0] (5.36,-2.57) -- (0,0) -- (5.36,2.57) -- (3.56,0) -- cycle    ;
      \node at (196, 107) {\footnotesize{$a_1(f(1))$}};

      \draw    (323,104.57) .. controls (333.98,101.96) and (337.78,102.93) .. (344.81,109.77) ;
      \draw [shift={(346.86,111.83)}, rotate = 225.87] [fill={rgb, 255:red, 0; green, 0; blue, 0 }  ][line width=0.08]  [draw opacity=0] (5.36,-2.57) -- (0,0) -- (5.36,2.57) -- (3.56,0) -- cycle    ;
      \node at (296, 107) {\footnotesize{$a_2(f(2))$}};

    \end{tikzpicture}
    \caption{A depiction of the class $\La'$. Given $j$, the dashed line depicts the set $W_{\hat{e}_j}$ and the cross indicates the element $a_j(f(j))$.}\label{Fig:HardestProof}
  \end{figure}

  We proceed with the formal proof that $h$ $\Txt\G\Mon\Ex$-learns $\La'$. Let $L' \in \La'$ and $T' \in \Txt(L')$. We first show the $\Ex$-convergence and the monotonicity afterwards. For the former, we distinguish the following cases.
  \begin{enumerate}
    \itemin{1. Case:} For some $j$, we have $L' = W_{e_j}$. Let $n_0$ such that $\content(T[n]) \neq \emptyset$. Then, for $n \geq n_0$, there exists some $i$ such that $a_j(i) = \max(\content(T[n]))$. Thus, 
    $$h(T[n]) = \varphi_{\max(\content(T[n]))}(T[n]) = \varphi_{a_j(i)}(T[n]) = e_j.$$
    Hence, $h$ learns $W_{e_j}$ correctly.
    \itemin{2. Case:} We have $L' = W_e$. Let $n_0, k_0 \in \N$, with $n_0$ minimal, such that $\content(T[n_0]) \neq \emptyset$ and $\first(T[n_0]) \in \range(a_{k_0})$. Let $n_1 \geq n_0$ be minimal such that there exists $k> k_0$ such that $\content(T[n_1])$ also contains elements from $\content(a_k)$. Then, for $n > n_1$ we have that $h(T[n]) = e$, as there exists no $j$ with $a_j(f(j)) \in \content(T)$ and also $\max \{ j \mid \content(T[n]) \cap \range(a_j) \neq \emptyset \} \neq k_0$. Thus, $h$ learns $W_{e}$ correctly.
    \itemin{3. Case:} For some $k>0$ we have $L' = W_{\hat{e}_k}$. In this case, there exists $n_0$ with $a_k(f(k)) \in \content(T[n_0])$. Then, for $n \geq n_0$, we have $h(T[n]) = \hat{e}_k$. Therefore, $h$ learns $W_{\hat{e}_k}$ correctly.
  \end{enumerate}
  We show that the learning is monotone. Let $n \in \N$. As long as $\content(T[n])$ is empty, $h$ returns $\ind(\emptyset)$. Once $\content(T[n])$ is not empty anymore and as long as $\content(T[n])$ only contains elements from, for some $j$, $\range(a_j)$, the learner $h$ outputs (a code for) the set $W_{e_j}$. Note that $j$ is the index of the element $\first(T[n])$, that is, $\first(T[n]) \in \range(a_j)$. If ever, for some later $n$, $\content(T[n]) \setminus \range(a_j) \neq \emptyset$, then $h$ only changes its mind if there exists $k > j$ such that $\content(T[n]) \cap \range(a_k) \neq \emptyset$. Depending on whether $a_k(f(k)) \in \content(T[n])$ or not, $h$ changes its mind to (a code of) either $W_{\hat{e}_k}$ or $W_e$, respectively. In the former case, the learner $h$ is surely presented the set $W_{\hat{e}_k}$, making this mind change monotone. In the latter, no element of $W_{e_j} \setminus \content(a_j(f(j)))$ is contained the target language. This are exactly the elements $h$ discards from its hypothesis, keeping a monotone behaviour. The learner only changes its mind again if it witnesses, for some $k' \geq k$, the element $a_{k'}(f(k'))$. It will then output (a code of) the set $W_{\hat{e}_{k'}}$. This is, again, monotonic behaviour, as $h$ is sure to be presented the set $W_{\hat{e}_{k'}}$. Altogether, $h$ is monotone on any text of $L$.

  Thus, $h$ identifies all languages in $\La'$ correctly. Now, we show that $h'$ cannot do so too. We do so by providing a text of $W_e$ where $h'$ makes infinitely many wrong guesses. To that end, consider the text $T$ of $W_e$ given as
  \[
    a_0[f(0)] a_1[f(1)] a_2[f(2)]\dots
  \]
  For $j>0$, since $a_j(f(j)) \in W_{h'(\content(a_j[f(j)]), f(j))}$, we have 
  \begin{align*}
    a_j(f(j)) \in W_{h'(\content(T[\sum_{m \leq j} f(m)]), \sum_{m \leq j} f(m))},
  \end{align*}
  as $T[\sum_{m \leq j} f(m)]$ is a initial sequence for a text for $W_{\hat{e}_j}$. But, since $a_j(f(j)) \notin W_e$, $h'$ makes infinitely many incorrect conjectures and thus does not identify $W_e$ on the text $T$ correctly, a contradiction. 
\end{proof}

Together with the results of \citet{KSS17} that monotone $\Bc$-learners may be assumed total and that partially set-driven monotone (explanatory) learners are more powerful than set-driven behaviourally correct ones, we completed the extension of the results of \citet{KS16}, see Figure~\ref{FigSMon}. We depict these results in Figure~\ref{FigMonBc}.

\section{Conclusion and Future Work} \label{Sec:Concl}

When given a learning task, monotonic learners display different behaviour depending on the particular setting. Building on the studies of \citet{KS16}, who unveil a peculiar behaviour of strongly monotone learners under various additional constraints when learning arbitrary classes of languages, we show the similarities and differences when considering monotone learners. Besides memory restrictions, we impose requirements, such as totality, on the learners themselves. The most notable similarity is that globally monotone learners are, in fact, also globally strongly monotone. Besides that, both learning types show a similar overall-picture when requiring syntactic convergence. However, the results and, thus, the picture drastically changes when requiring semantic convergence. We show that monotone behaviourally correct learners only achieve their full learning power when having full information to infer their guesses from, that is, partially set-driven monotone behaviourally correct learners are strictly less powerful than their Gold-style counterpart. For behaviourally correct learners, this is a novelty.

The desire to discover more such novelties strengthens the need to further investigate monotonic restrictions. In particular, obtaining an overview of the situation regarding \emph{weakly monotone} learners \citep{Jantke91,Wiehagen91}, which need to be strongly monotone while being consistent, can be considered the next natural step.

\acks{%
  This work was supported by DFG Grant Number KO 4635/1-1. 
}

\bibliography{LTbib}

\end{document}